\documentclass{article} 
\usepackage{nips11submit_e}
\usepackage{times}
\usepackage{graphicx} 
\usepackage{subfig} 
\usepackage{amssymb}
\usepackage{amsmath}
\usepackage{amsthm}
\usepackage{epstopdf}
\usepackage[disable]{todonotes}
\usepackage{natbib}
\usepackage{algorithm}
\usepackage{algorithmic}
\usepackage{appendix}
\usepackage{textcomp}
\usepackage{multirow}
\usepackage{url}
\usepackage{hyperref}

\newtheorem{lemma}{Lemma}

\newtheorem{proposition}[lemma]{Proposition}

\DeclareMathOperator*{\argmax}{arg\,max}
\DeclareMathOperator*{\argmin}{arg\,min}

\def\transpose{{^\top}}

\title{Alignment Based Kernel Learning with a Continuous Set of Base Kernels}

\author{
Arash Afkanpour \hspace{10mm} Csaba Szepesv\'{a}ri \hspace{10mm} Michael Bowling	\\ \\
Department of Computing Science\\
University of Alberta\\
Edmonton, AB T6G 1K7\\
\{\texttt{afkanpou,szepesva,mbowling}\}\texttt{@ualberta.ca} \\
}

\nipsfinalcopy 

\begin{document}

\maketitle

\begin{abstract}
The success of kernel-based learning methods depend on the choice of kernel. Recently, kernel learning methods have been proposed that use data to select the most appropriate kernel, usually by combining a set of base kernels. We introduce a new algorithm for kernel learning that combines a {\em continuous set of base kernels}, without the common step of discretizing the space of base kernels. We demonstrate that our new method achieves state-of-the-art performance across a variety of real-world datasets. Furthermore, we explicitly demonstrate the importance of combining the right dictionary of kernels, which is problematic for methods based on a finite set of base kernels chosen a priori. Our method is not the first approach to work with continuously parameterized kernels. However, we show that our method requires substantially less computation than previous such approaches, and so is more amenable to multiple dimensional parameterizations of base kernels, which we demonstrate.  
\end{abstract}

\section{Introduction}

A well known fact in machine learning is that the choice of features heavily influences the performance of learning methods.
Similarly, the performance of a learning method that uses a kernel function is highly dependent on the choice of kernel function.
The idea of {\em kernel learning} 
is to use data to select the most appropriate kernel function for the learning task.

In this paper we consider kernel learning in the context of supervised learning.
In particular, we consider the problem of learning positive-coefficient linear combinations of base kernels,
 where the base kernels belong to a parameterized family of kernels, $(\kappa_\sigma)_{\sigma\in \Sigma}$.
Here $\Sigma$ is a ``continuous'' parameter space, i.e., some subset of a Euclidean space.
A prime example (and extremely popular choice) is when $\kappa_\sigma$ is a Gaussian kernel, where
 $\sigma$ can be a single common bandwidth or a vector of bandwidths, one per coordinate.
One approach then is to discretize the parameter space $\Sigma$ and then find an appropriate non-negative linear combination of the resulting set of base kernels, ${\cal N} = \{\kappa_{\sigma_1},\ldots,\kappa_{\sigma_p}\}$. The advantage of this approach is that once the set ${\cal N}$ is fixed, any of the many efficient methods available in the literature can be used to find the coefficients for combining the base kernels in ${\cal N}$ (see the papers by \citealt{lanckriet2004learning, sonnenburg2006large,rakotomamonjy2008simplemkl,cortes2009l2,kloft2011lp} and the references therein).
One potential drawback of this approach is that it requires an appropriate, {\em a priori} choice of  ${\cal N}$.
This might be problematic, e.g., if $\Sigma$ is contained in a Euclidean space of moderate, or large dimension (say, a dimension over 20) since the number of base kernels, $p$, grows exponentially with dimensionality even for moderate discretization accuracies. Furthermore, independent of the dimensionality of the parameter space, the need to choose the set ${\cal N}$ independently of the data is at best inconvenient and selecting an appropriate resolution might be far from trivial.
In this paper we explore an alternative method which avoids the need for discretizing the space $\Sigma$.

We are not the first to realize that discretizing a continuous parameter space might be troublesome:
The method of \citet{argyriou2005learning,argyriou2006dc} can also work with continuously parameterized spaces of kernels. \todo{Note that  Eq. (6) of \citet{argyriou2006dc} is convex in $K$. Hence, we could use, in the future, our gradient techniques directly on this performance measure.}
The main issue with this method, however, is that it may get stuck in local optima since it is based on alternating minimization and the objective function is not jointly convex.
Nevertheless, empirically, in the initial publications of \citet{argyriou2005learning,argyriou2006dc} this method was found to have excellent and robust performance, showing that despite the potential difficulties, the idea of avoiding discretizations might have some traction.

Our new method is similar to that of \citet{argyriou2005learning,argyriou2006dc}, in that it is still based on local search. However, our local search is used within a boosting, or more precisely, forward-stagewise additive modeling (FSAM) procedure, a method that is known to be quite robust to how its ``greedy step'' is implemented \citep[Section 10.3]{hastie2001elements}. Thus, we expect to suffer minimally from issues related to local minima.
A second difference to  \citet{argyriou2005learning,argyriou2006dc} is that our method belongs to the group of {\em two-stage kernel learning} methods. The decision to use a two-stage kernel learning approach was motivated by the recent  success of the two-stage method of  \citet{cortes2010two}. In fact, our kernel learning method uses  the centered  kernel alignment metric of \citet{cortes2010two} (derived from the uncentered alignment metric of \citet{shawe2002kernel}) in its first stage as the objective function of the FSAM procedure, while in the second stage a standard supervised learning technique is used.

The technical difficulty of implementing FSAM is that one needs to compute the functional gradient of the chosen objective function. We show that in our case this problem is equivalent to solving an optimization problem over $\sigma\in\Sigma$ with
an objective function that  is a linear function of the Gram matrix derived from the kernel $\kappa_\sigma$.
Because of the nonlinear dependence of this matrix on $\sigma$, this is the step where we need to resort to local optimization: this optimization problem is in general non-convex. However, as we shall demonstrate empirically, even if we use local solvers to solve this optimization step, the  algorithm still shows an overall excellent performance as compared to other state-of-the-art methods. This is not completely unexpected: One of the key ideas underlying boosting is that it is designed to be robust even when the individual ``greedy'' steps are imperfect (cf., Chapter 12, \citealt{BuhvdG11}). \todo{If we used linear kernels, Euclidean distance learning, this step would also be convex. Could this be interesting!?} Given the new kernel to be added to the existing dictionary, we give a computationally efficient, closed-form expression that can be used to determine the coefficient on the new kernel to be added to the previous kernels.

The empirical performance of our proposed method is explored in a series of experiments.
Our experiments serve multiple purposes.
Firstly, we explore the potential advantages, as well as limitations of the proposed technique. 
In particular, we demonstrate that the procedure is indeed reliable (despite the potential difficulty of implementing the greedy step) and that it can be successfully used even when $\Sigma$ is a subset of a multi-dimensional space. 
Secondly, we demonstrate that in some cases, kernel learning can have a very large improvement over simpler alternatives, such as combining some fixed dictionary of kernels with uniform weights. 
Whether this is true is an important issue that is given weight by the fact that just recently it became a subject of dispute \citep{Cortes09:invited}.
Finally, we compare the performance of our method, both from the perspective of its generalization capability and computational cost, to its natural, state-of-the-art alternatives, such as the two-stage method of \citet{cortes2010two} and the algorithm of
\citet{argyriou2005learning,argyriou2006dc}. For this, we compared our method on datasets used in previous kernel-learning work. To give further weight to our results, we compare on more datasets than any of the previous papers that proposed new kernel learning methods.

Our experiments demonstrate  that our new method is {\em competitive in terms of its generalization performance, while its computational cost is significantly less than that of its competitors that enjoy similarly good generalization performance as our method}.
In addition, our experiments also revealed an interesting novel insight into the behavior of two-stage methods: we noticed that two-stage methods can ``overfit'' the performance metric of the first stage.  In some problem we observed that our method could find kernels that gave rise to better (test-set) performance on the first-stage metric, while the method's overall performance degrades when compared to using kernel combinations whose performance on the first metric is worse. The explanation of this is  that metric of the first stage  is a surrogate performance measure and thus just like in the case of choosing a surrogate loss in classification, better performance according to this surrogate metric does not necessarily transfer into better performance in the primary metric as there is no monotonicity relation between these two metrics. We also show that with proper capacity control, the problem of overfitting the surrogate metric can be overcome. Finally, our experiments  show a clear advantage to using kernel learning methods as opposed to combining kernels with a uniform weight, although it seems that the advantage mainly comes from the ability of our method to discover the right set of kernels. This conclusion is strengthened by the fact that the closest competitor to our method was found to be the method of \citet{argyriou2006dc} that also searches the continuous parameter space, avoiding discretizations. Our conclusion is that it seems that the choice of the base dictionary is more important than how the dictionary elements are combined \todo{Did we try say uniform weights on the top of the dictionary we found? If yes, is this noted somewhere in the paper? In the conclusion these things should be repeated.} and that the {\em a priori} choice of this dictionary may not be trivial. This is certainly true already when the number of parameters is moderate. Moreover, when the number of parameters is larger, simple discretization methods are infeasible, whereas our method can still produce meaningful dictionaries.

\newcommand{\ra}{\rightarrow}
\newcommand{\real}{\mathbb{R}}

\section{The New Method}
\label{sec:new_method}

The purpose of this section is to describe our new method.
Let us start with the introduction of the problem setting and the notation.
We consider binary classification problems, where the data ${\cal D} = ((X_1,Y_1),\ldots,(X_n,Y_n))$ is a sequence of independent, identically distributed random variables, with $(X_i,Y_i)\in \real^d\times \{-1,+1\}$.
For convenience, we introduce two other pairs of random variables $(X,Y)$, $(X',Y')$, which are also independent of each other and they share the same distribution with $(X_i,Y_i)$.
The goal of classifier learning is to find a predictor, $g:\real^d \ra \{-1,+1\}$ such that the predictor's risk, $L(g) = \mathbb{P}( g(X)\not=Y )$, is close to the Bayes-risk, $\inf_g L(g)$. We will consider a two-stage method, as noted in the introduction. The first stage of our method will pick some kernel $k:\real^d\times\real^d \ra \real$ from some set of kernels ${\cal K}$ based on ${\cal D}$, which is then used in the second stage, using the same data ${\cal D}$ to find a good predictor.\footnote{One could consider splitting the data, but we see no advantage to doing so. Also, the methods for the second stage are not a focus of this work and the particular methods used in the experiments are described later.}


Consider a parametric family of base kernels, $(\kappa_\sigma)_{\sigma\in \Sigma}$.
The  kernels considered by our method belong to the set
\[
\mathcal{K} = \left\{ \sum_{i=1}^r \mu_i \kappa_{\sigma_i} : r \in \mathbb{N}, \mu_i \geq 0, \sigma_i \in \Sigma, i=1,\ldots,r \right\},
\]
i.e., we allow non-negative linear combinations of a finite number of base kernels.
For example, the base kernel could be a Gaussian kernel, where $\sigma>0$ is its bandwidth:
$\kappa_\sigma(x,x') = \exp(-\|x-x'\|^2/\sigma^2)$, where $x,x'\in \real^d$.
However, one could also have a separate bandwidth for each coordinate.

\newcommand{\EE}[1]{\mathbb{E}\left[#1\right]}
\newcommand{\eqdef}{\stackrel{{\rm def}}{=}}
\newcommand{\ipL}[1]{\left\langle #1 \right\rangle}
\newcommand{\ip}[1]{\langle #1 \rangle}
\newcommand{\normL}[1]{\left\|#1\right\|}
\newcommand{\norm}[1]{\|#1\|}
The ``ideal'' kernel underlying the common distribution of the data
is $k^*(x,x') = \EE{\, YY' \,| \, X=x,X'=x'\,}$.
Our new method attempts to find a kernel $k \in {\mathcal{K}}$ which is maximally aligned to
this ideal kernel, where, following~\citet{cortes2010two}, the alignment between two kernels $k,\tilde{k}$ is measured by the {\em centered alignment metric},\footnote{Note that the word metric is used in its everyday sense and not in its mathematical sense.}
\[
A_c(k,\tilde{k}) \eqdef \frac{ \ip{k_c,\tilde{k}_c} }{ \norm{k_c}\norm{\tilde{k}_c} }\,,
\]
where
$k_c$ is the kernel underlying $k$ centered in the feature space (similarly for $\tilde{k}_c$),
$\ip{k,\tilde{k}} = \EE{ k(X,X') \tilde{k}(X,X') }$ and $\norm{k}^2 = \ip{k,k}$.
A kernel $k$ centered in the feature space, by definition, is the unique kernel  $k_c$,
 such that for any $x,x'$,
$k_c(x,x') = \ip{ \Phi(x) - \EE{\Phi(X)}, \Phi(x') - \EE{ \Phi(X) } }$, where $\Phi$ is a feature map underlying $k$.
By considering centered kernels $k_c$, $\tilde{k}_c$ in the alignment metric,
 one implicitly matches the mean responses $\mathbb{E}[k(X,X')]$, $\mathbb{E}[\tilde{k}(X,X')]$ before considering the alignment between the kernels (thus, centering depends on the distribution of $X$).
An alternative way of stating this is that centering cancels mismatches of the mean responses between the two kernels.
When one of the kernels is the ideal kernel, centered alignment effectively standardizes the alignment by cancelling the effect of imbalanced class distributions.
For further discussion of the virtues of centered alignment, see the paper by \citet{cortes2010two}.

Since the common distribution underlying the data is unknown, one resorts to empirical approximations to alignment and centering, resulting in the empirical alignment metric,
\[
A_c(K,\tilde{K}) = \frac{ \ip{K_c,\tilde{K}_c}_F }{ \norm{K_c}_F\norm{\tilde{K}_c}_F }\,,
\]
where, $K=(k(X_i,X_j))_{1\le i,j\le n}$, and $\tilde{K} = (\tilde{k}(X_i,X_j))_{1\le i,j\le n}$ are the kernel matrices underlying $k$ and $\tilde{k}$, and for a kernel matrix, $K$, $K_c = C_n K C_n$, where $C_n$ is the so-called centering matrix defined by $C_n = I_{n\times n} - \mathbf{1} \mathbf{1}^\top /n$, $I_{n\times n}$ being the $n\times n$ identity matrix and $\mathbf{1} = (1,\ldots,1)^\top \in \real^n$.
\newcommand{\bY}{\mathbf{Y}}
The empirical counterpart of maximizing $A_c(k,k^*)$ is to maximize $A_c(K,\hat{K}^*)$, where $\hat{K}^* \eqdef \bY \bY^T$, and $\bY  =(Y_1,\ldots,Y_n)^\top$ collects the responses into an $n$-dimensional vector.
Here, $K$ is the kernel matrix derived from a kernel $k \in {\cal K}$. To make this connection clear, we will write $K=K(k)$.
Define $f:{\mathcal K} \ra \real$ by $f(k) = A_c( K(k), \hat{K}^* )$.

To find an approximate maximizer of $f$, we propose a steepest ascent approach to {\em forward stagewise additive modeling} (FSAM). FSAM~\citep{hastie2001elements} is an iterative method for optimizing an objective function by sequentially adding new basis functions without changing the parameters and coefficients of the previously added basis functions.
In the steepest ascent approach, in iteration $t$, we search for the base kernel in $(\kappa_\sigma)$ defining the direction in which the growth rate of $f$ is the largest, locally in a small neighborhood of the previous candidate $k^{t-1}$:
\newcommand{\eps}{\varepsilon}
\renewcommand{\epsilon}{\varepsilon}
\begin{equation}
\label{eq:sigmatdef}
\sigma^*_t = \argmax_{\sigma\in \Sigma}
\lim_{\eps\ra 0} \frac{f(k^{t-1}+\eps\, \kappa_\sigma)-f(k^{t-1})}{\eps}\,.
\end{equation}
Once $\sigma^*_t$ is found, the algorithm finds the coefficient $0\le \eta_t \le \eta_{\max}$\footnote{In all our experiments we use the arbitrary value $\eta_{\max}=1$. Note that the value of $\eta_{\max}$, together with the limit $T$ acts as a regularizer. However, in our experiments, the procedure always stops before the limit $T$ on the number of iterations is reached.} such that $f(k^{t-1}+\eta_t \kappa_{\sigma_t^*})$ is maximized and the candidate is updated using $k^t = k^{t-1}+\eta_t \kappa_{\sigma_t^*}$.
The process stops  when  the objective function $f$ ceases to increase by an amount larger than $\theta>0$, or when the number of iterations becomes larger then a predetermined limit $T$, whichever happens earlier.
\begin{algorithm}[tb]
\caption{Forward stagewise additive modeling for
 kernel learning with a continuously parametrized set of kernels.
 For the definitions of $f$, $F$, $F'$ and $K:{\mathcal K} \ra \real^{n\times n}$, see the text.
}
\label{alg:continuous_method}
\begin{algorithmic}[1]
    \STATE \textbf{Inputs:} data ${\cal D}$, kernel initialization parameter $\varepsilon$, the number of iterations $T$,
     tolerance $\theta$, maximum stepsize $\eta_{\max}>0$.
    \STATE $K^0 \leftarrow \varepsilon I_n$.
    \FOR{$t=1$ {\bfseries to} $T$}
			\STATE $P \leftarrow F'(K^{t-1})$ 
			\STATE $P \leftarrow C_n\, P\, C_n$
			\STATE $\sigma^* = \argmax_{\sigma\in \Sigma}\, \ip{ P, K(\kappa_\sigma) }_F$ 
			\STATE $K' =  C_n\, K (\kappa_{\sigma^*})\, C_n$
			\STATE $\eta^* = \argmax_{0\le \eta \le \eta_{\max}} F(K^{t-1} + \eta K' )$
			\STATE $K^t \leftarrow K^{t-1} + \eta^*K'$
 			\STATE \textbf{if} $F(K^t) \leq F(K^{t-1})+\theta$ \textbf{then} terminate
		\ENDFOR
\end{algorithmic}
\end{algorithm}

\begin{proposition}\label{prop:dd}
The value of $\sigma_t^*$ can be obtained by
\begin{eqnarray}
\sigma_t^* = \argmax_{\sigma\in \Sigma} \,
	\ipL{\, K(\kappa_\sigma), F'(\, (K(k^{t-1}))_c ) \, }_F\,,
\label{eq:sigma_star}
\end{eqnarray}
where for a kernel matrix $K$,
\begin{align}
\label{eq:Fder}
F'(K) = \frac{ \hat{K}^*_c - \norm{K}_F^{-2}\ip{ K, \hat{K}^*_c}_F\, K }
				   { \norm{K}_F \norm{\hat{K}^*_c}_F}\,.
\end{align}
\end{proposition}
The proof can be found in the supplementary material. The crux of the proposition is that the directional derivative in~\eqref{eq:sigmatdef} can be calculated and gives the expression maximized in~\eqref{eq:sigma_star}.

In general, the optimization problem~\eqref{eq:sigma_star} is not convex and the cost of obtaining a (good approximate) solution is hard to predict.  Evidence that, at least in some cases, the function to be optimized is not ill-behaved is presented in Section~\ref{sec:non_convexity} of the supplementary material. In our experiments, an approximate solution to~\eqref{eq:sigma_star} is found using numerical methods.\footnote{
In particular, we use the \textsf{fmincon} function of Matlab, with the interior-point algorithm option.
}
As a final remark to this issue, note that, as is usual in boosting, finding the global optimizer in~\eqref{eq:sigma_star} might not be necessary for achieving good statistical performance.


The other parameter, $\eta_t$, however, is easy to find, since the underlying optimization problem has a closed form solution:
\begin{proposition}\label{prop:pp2}
The value of $\eta_t$ is given by $\eta_t=\argmax_{\eta\in \{0,\eta^*,\eta_{\max}\}} f(k^{t-1}+\eta \kappa_{\sigma_t^*} )$,
where $\eta^*=\max(0,(ad-bc)/(bd-ae))$ if $bd-ae\not=0$ and $\eta^*=0$ otherwise, $a=\ip{K,\hat{K}^*_c}_F$, $b=\ip{K',\hat{K}^*_c}_F$, $c=\ip{K,K}_F$, $d=\ip{K,K'}_F$, $e=\ip{K',K'}_F$ and  $K= (K(k^{t-1}))_c$, $K'=(K( \kappa_{\sigma_t^*} ))_c$.
\end{proposition}

The pseudocode of the full algorithm is presented in Algorithm~\ref{alg:continuous_method}.
The algorithm needs the data, the number of iterations  ($T$) and a tolerance ($\theta$) parameter, in addition to a parameter $\eps$ used in the initialization phase and $\eta_{\max}$. The parameter $\eps$ is used in the initialization step to avoid division by zero, and its value has little effect on the performance. Note that the cost of computing a kernel-matrix, or the inner product of two such matrices is $O(n^2)$. Therefore, the complexity of the algorithm (with a naive implementation) is at least quadratic in the number of samples. The actual cost will be strongly influenced by how many of these kernel-matrix evaluations (or inner product computations) are needed in~\eqref{eq:sigma_star}. In the lack of a better understanding of this, we include actual running times in the experiments, which give a rough indication of the computational limits of the procedure.

\section{Experimental Evaluation}
\label{sec:experimental_results}

In this section we compare our kernel learning method with several kernel learning methods on synthetic and real data; see Table \ref{tab:list_of_methods} for the list of methods.
 Our method is labeled CA for Continuous Alignment-based kernel learning. In all of the experiments, we use the following values with CA: $T = 50$, $\varepsilon = 10^{-10}$, and $\theta = 10^{-3}$.
The first two methods, i.e. our algorithm, and CR~\citep{argyriou2005learning}, are able to pick kernel parameters from a continuous set, while the rest of the algorithms work with a finite number of base kernels.

In Section \ref{sec:synthetic_data_experiment} we use synthetic data to illustrate the potential advantage of methods that work with a continuously parameterized set of kernels and the importance of combining multiple kernels. We also illustrate in a toy example that multi-dimensional kernel parameter search can improve performance. These are followed by the evaluation of the above listed methods on several real datasets in Section \ref{sec:real_data_experiment}.

\begin{table}[t]
\caption{List of the kernel learning methods evaluated in the experiments. The key to the naming of the methods is as follows:
CA stands for ``continuous alignment'' maximization, CR stands for ``continuous risk'' minimization, DA stands for ``discrete alignment'', D1, D2, DU should be obvious.
}
\label{tab:list_of_methods}
\begin{center}
\begin{tabular}{|l|l|}
  \hline
	\textbf{Abbr.} & \textbf{Method} \\ \hline
	CA  & Our new method\\ \hline
	CR  & From \citet{argyriou2005learning} \\ \hline
	DA  & From \citet{cortes2010two} \\ \hline
	D1  & $\ell_1$-norm MKL \citep{kloft2011lp} \\ \hline
	D2  & $\ell_2$-norm MKL \citep{kloft2011lp} \\ \hline
	DU  & Uniform weights over kernels \\ \hline
\end{tabular}
\end{center}
\end{table}

\subsection{Synthetic Data}
\label{sec:synthetic_data_experiment}

The purpose of these experiments is mainly to provide empirical proof for the following hypotheses:
{(H1)} The combination of multiple kernels can lead to improved  performance as compared to what can be achieved with a single kernel, even when in theory a single kernel from the family suffices to get a consistent classifier.
{(H2)} The methods that search the continuously parameterized families are able to find the ``key'' kernels and their combination.
{(H3)} Our method can even search multi-dimensional parameter spaces, which in some cases is crucial for good performance.

To illustrate (H1) and (H2) we have designed the following problem:
the inputs are generated from the uniform distribution over the interval $[-10,10]$. The label of each data point is determined by the function $y(x) = \textrm{sign}(f(x))$, where $f(x) = \sin(\sqrt{2}x) + \sin(\sqrt{12}x) + \sin(\sqrt{60}x)$.
Training and validation sets include $500$ data points each, while the test set includes $1000$ instances.
Figure~\ref{fig:dirichlet_synthetic}(a) shows the functions $f$ (blue curve) and $y$ (red dots).
For this experiment we use Dirichlet kernels of degree one,\footnote{We repeated the experiments using Gaussian kernels with nearly identical results.}
 parameterized with a frequency parameter $\sigma$: $\kappa_{\sigma}(x,x') = 1 + 2\cos(\sigma \|x-x'\|)$.

In order to  investigate (H1), we trained classifiers
with a single frequency kernel from the set $\sqrt{2}$, $\sqrt{12}$, and $\sqrt{60}$ (which we thought were good guesses of the single best frequencies).
The trained classifiers achieved misclassification error rates of $26.1\%$, $26.8\%$, and $28.6\%$, respectively. Classifiers trained with a pair of frequencies, i.e. $\{\sqrt{2},\sqrt{12}\}$, $\{\sqrt{2},\sqrt{60}\}$, and $\{\sqrt{12},\sqrt{60}\}$ achieved error rates of $16.4\%$, $20.0\%$, and $21.3\%$, respectively (the kernels were combined using uniform weights). Finally, a classifier that was trained with all three frequencies achieved an error rate of $2.3\%$.

Let us now turn to (H2). As shown in Figure~\ref{fig:dirichlet_synthetic}(b), the CA and CR methods both achieved a misclassification error close to what was seen when the three best frequencies were used, showing that they are indeed effective.\footnote{In all of the experiments in this paper, the classifiers for the two-stage methods were trained using the soft margin SVM method, where the  regularization coefficient of SVM was chosen by cross-validation from $10^{\{-5,-4.5,\ldots,4.5,5\}}$.}
Furthermore, Figure~\ref{fig:dirichlet_synthetic}(c) shows that the discovered frequencies are close to the frequencies used to generate the data.
For the sake of illustration, we also tested the methods which require the discretization of the parameter space. We choose  ten Dirichlet kernels with $\sigma \in \{0,1,\ldots,9\}$, covering the range of frequencies defining $f$.
As can be seen from Figure~\ref{fig:dirichlet_synthetic}(b) in this example the chosen discretization accuracy is insufficient. Although it would be easy to increase the discretization accuracy to improve the results of these methods,\footnote{Further experimentation found that a discretization below $0.1$ is necessary in this example.} the point is that if a high resolution is needed in a single-dimensional problem, then these methods are likely to face serious difficulties in problems when the space of kernels is more complex (e.g., the parameterization is multidimensional). Nevertheless, we are not suggesting that the methods which require discretization are universally inferior, but merely wish to point out that an ``appropriate discrete kernel set'' might not always be available.
\begin{figure}[htbp]
\vspace{.3in}
\centerline{\includegraphics[width=0.5\linewidth]{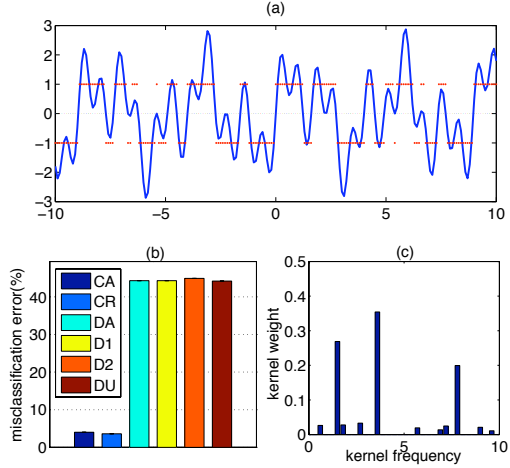}}
\vspace{.3in}
\caption{(a): The function $f(x) = \sin(\sqrt{2}x) + \sin(\sqrt{12}x) + \sin(\sqrt{60}x)$ used for generating synthetic data, along with $\text{sign}(f)$. (b): Misclassification percentages obtained by each algorithm. (c): The kernel frequencies found by the CA method.}
\label{fig:dirichlet_synthetic}
\end{figure}

To illustrate (H3) we designed a second set of problems:
The instances for the positive (negative) class are generated from a $d=50$-dimensional Gaussian distribution with  covariance matrix $C = I_{d \times d}$ and mean $\mu_1 = \rho \frac{\theta}{\|\theta\|}$ (respectively, $\mu_2 = -\mu_1$ for the negative class). Here $\rho=1.75$. \todo{(Mike) I don't understand what the instance is.  So you have two 50-dimensional Gaussians. What's X and Y?}  
The vector $\theta \in [0,1]^d$ determines the relevance of each feature in the classification task, e.g. $\theta_i=0$ implies that the distributions of the two classes have zero means in the $i$th feature, which renders this feature irrelevant. The value of each component of vector $\theta$ is calculated as $\theta_i = (i/d)^\gamma$, where $\gamma$ is a constant that determines the relative importance of the elements of $\theta$. We generate seven datasets with $\gamma \in \{0,1,2,5,10,20,40\}$.
For each value of $\gamma$, the training set consists of $50$ data points (the prior distribution for the two classes is uniform). The test error values are measured on a test set with $1000$ instances. We repeated each experiment $10$ times and report the average misclassification error and alignment measured over the test set along with the running time.

We test two versions of our method: one that uses a family of Gaussian kernels with a common bandwidth (denoted by CA-1D), and another one (denoted by CA-nD) that searches in the space $(\kappa_\sigma)_{\sigma\in (0,\infty)^{50}}$, where each coordinate has a separate bandwidth parameter, $\kappa_{\sigma}(x,x') = \exp(-\sum_{i=1}^{d}  (x_i-x_i')^2/\sigma_i^{2} )$. Since the training set is small, one can easily overfit while optimizing the alignment. Hence, we modify the algorithm to shrink the values of the bandwidth parameters to their common average value by modifying~\eqref{eq:sigma_star}:
\begin{eqnarray}
\sigma_t^* = \argmin_{\sigma\in \Sigma} \,
	&& -\ipL{\, K(\kappa_\sigma), F'(\, (K(k^{t-1}))_c ) \, }_F\, \nonumber \\ && + \lambda  \| \sigma - \bar{\sigma} \|_2^2,
\label{eq:new_sigma_star}
\end{eqnarray}
where, $\bar{\sigma} = \frac{1}{r} \sum_{i=1}^r \sigma_i$ and $\lambda$ is a regularization parameter.
We also include results obtained for
 finite kernel learning methods. For these methods, we generate $50$ Gaussian kernels with bandwidths $\sigma \in m g^{\{0,\ldots,49\}}$, where $m=10^{-3}$, and $g\approx1.33$. Therefore, the bandwidth range constitutes a geometric sequence from $10^{-3}$ to $10^3$.
 Further details of the experimental setup can be found in Section~\ref{apx:50d} of the supplementary material.

Figure \ref{fig:my_toy_experiment} shows the results. Recall that the larger the value of $\gamma$, the larger the number of nearly irrelevant features. Since methods which search only a one-dimensional space cannot differentiate between relevant and irrelevant features, their misclassification rate increases with $\gamma$. Only CA-nD is able to cope with this situation and even improve its performance.
We observed that without regularization, though, CA-nD drastically overfits (for small values of $\gamma$).
We also show the running times of the methods to give the reader an idea about the scalability of the methods. The running time of CA-nD is larger than CA-1D both because of the use of cross-validation to tune $\lambda$ and because of the increased cost of the multidimensional search. Although the large running time might be a problem, for some problems, CA-nD might be the only method to deliver good performance amongst the methods studied.\footnote{We have not attempted to run a multi-dimensional version of the CR method, since already the one-dimensional version of this method is at least one order of magnitude slower than our CA-1D method.}
\begin{figure*}[htbp]
\vspace{.3in}
\centerline{\includegraphics[width=0.75\linewidth]{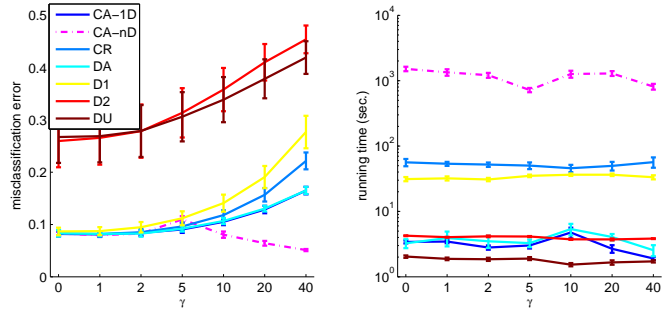}}
\vspace{.3in}
\caption{Performance and running time of various methods for a $50$-dimensional synthetic problem as a function of the relevance parameter $\gamma$. Note that the number of irrelevant features increases with $\gamma$. For details of the experiments, see the text.}
\label{fig:my_toy_experiment}
\end{figure*}

\subsection{Real Data}
\label{sec:real_data_experiment}

We evaluate the methods listed in Table \ref{tab:list_of_methods} on several binary classification tasks from MNIST and the UCI Letter recognition dataset, along with several other datasets from the UCI machine learning repository \citep{frank2010uci} and Delve datasets (see, \url{http://www.cs.toronto.edu/~delve/data/datasets.html}).

\paragraph{MNIST.}
In the first experiment, following~\citet{argyriou2005learning}, we choose 8 handwritten digit recognition tasks of various difficulty from the MNIST dataset \citep{lecun-mnist-web-2010}. This dataset consists of $28\times28$ images with pixel values ranging between $0$ and $255$. In these experiments, we used Gaussian kernels with parameter $\sigma$: $G_\sigma(x,x') = \exp(-\|x-x'\|^2/\sigma^2)$. Due to the large number of attributes (784) in the MNIST dataset, we only evaluate the 1-dimensional version of our method. For the algorithms that work with a finite kernel set, we pick 20 kernels with the value of $\sigma$ picked from an equidistant discretization of interval $[500,50000]$.
In each experiment, the training and validation sets consist of $500$ and $1000$ data points, while the test set has $2000$ data points.
We repeated each experiment 10 times. Due to the lack of space, the test-set error plots for all of the problems can be found in the supplementary material (see Section~\ref{apx:realdata}). In order to give an overall impression of the algorithms' performance, we ranked them based on the results obtained in the above experiment. Table~\ref{tab:median_rank_and_time} reports the median ranks of the methods for the experiment just described.

Overall, methods that choose $\sigma$ from a continuous set outperformed their finite counterparts. This suggests again that for the finite kernel learning methods the range of $\sigma$ and the discretization of this range is important to the accuracy of the resulting classifier.

\begin{table*}[t]
\caption{Median rank and running time (sec.) of kernel learning methods obtained in experiments.}
\label{tab:median_rank_and_time}
\begin{center}
\begin{tabular}{c l|c|c|c|c|c|c|c|}
					&	 & CA-1D & CA-nD & CR & DA  & D1  & D2 & DU  \\ \hline
	\multirow{3}{*}{Rank} &MNIST & 1 & N/A &  2  & 4.5 & 4.5 & 5  & 4   \\
												&Letter& 1 & 4.5 & 2  & 3.5   & 7   & 6  & 5   \\
												&11 datasets& 3 & 2 & 3 & 3 & 4 & 6 & 6 \\ \hline \hline
	\multirow{2}{*}{Time} &MNIST & $12\pm1$ & N/A & $377\pm56$ & $31\pm1$ & $57\pm6$ & $58\pm3$ & $10\pm1$ \\
												&Letter& $9\pm1$ & $1986\pm247$ & $590\pm21$ & $11\pm1$ & $21\pm1$ & $22\pm1$ & $5\pm1$ \\ \hline
\end{tabular}
\end{center}
\end{table*}

\paragraph{UCI Letter Recognition.}
In another experiment, we evaluated these methods on 12 binary classification tasks from the UCI Letter recognition dataset. This dataset includes $20000$ data points of the 26 capital letters in the English alphabet.
For each binary classification task, the training and validation sets include $300$ and $200$ data points, respectively. The misclassification errors are measured over $1000$ test points. As with MNIST, we used Gaussian kernels. However, in this experiment, we ran our method with both 1-dimensional and $n$-dimensional search procedures. The rest of the methods learn a single parameter and the finite kernel learning methods were provided with 20 kernels with $\sigma$'s chosen from the interval $[1,200]$ in an equidistant manner. The plots of misclassification error and alignment are available in the supplementary material (see Section~\ref{apx:realdata}). We report the median rank of each method in Table \ref{tab:median_rank_and_time}. While the 1-dimensional version of our method outperforms the rest of the methods, the classifier built on the kernel found by the multi-dimensional version of our method did not perform well.  We examined the value of alignment between the learned kernel and the target label kernel on the test set achieved by each method. The results are available in the supplementary material (see Section~\ref{apx:realdata}). The multidimensional version of our method achieved the highest value of alignment in every task in this experiment.  Higher value of alignment between the learned kernel and the ideal kernel does not necessarily translate into higher value of accuracy of the classifier.  Aside from this observation, the same trends observed in the MNIST data can be seen here. The continuous kernel learning methods (CA-1D and CR) outperform the finite kernel learning methods.

\paragraph{Miscellaneous datasets.}
In the last experiment we evaluate all methods on 11 datasets chosen from the UCI machine learning repository and Delve datasets. Most of these datasets were used previously to evaluate kernel learning algorithms \citep{lanckriet2004learning,cortes2009l2,cortes2009learning,cortes2010two,rakotomamonjy2008simplemkl}. The specification of each dataset and the performance of each method are available in the supplementary material (see Section~\ref{apx:realdata}). The median rank of each method is shown in Table \ref{tab:median_rank_and_time}. Contrary to the Letter experiment, in this case the multi-dimensional version of our method outperforms the rest of the methods.


\paragraph{Running Times.}  We measured the time required for each run and each kernel learning method in the MNIST and the UCI Letter experiments. In each case we took the average of the running time of each method over all tasks. The average required time along with the standard error values are shown in Table \ref{tab:median_rank_and_time}. Among all methods, the DU method is fastest, which is expected, as it requires no additional time to compute kernel weights. The CA-1D is the fastest among the rest of the methods. In these experiments our method converges in less than 10 iterations (kernels). The general trend is that one-stage kernel learning methods, i.e., D1, D2, and CR, are slower than two-stage methods, CA and DA. Among all methods, the other continuous kernel learning method, CR, is slowest, since (1) it is a one-stage algorithm and (2) it usually requires more iterations (around $50$) to converge. We also examined the DC-Programming version of the CR method \cite{argyriou2006dc}. While it is faster than the original gradient-based approach (roughly three times faster), it is still significantly slower than the rest of the methods in our experiments. 

\section{Conclusion and Future Work}

We presented a novel method for kernel learning. This method addresses the problem of learning a kernel in the positive linear span of some continuously parameterized kernel family. The algorithm implements a steepest ascent approach to forward stagewise additive modeling to maximize an empirical centered correlation measure between the kernel and the empirical approximation to the ideal response-kernel.
The method was shown to perform well in a series of experiments, both with synthetic and real-data. We showed that in single-dimensional kernel parameter search, our method outperforms standard multiple kernel learning methods without the need to discretizing the parameter space. While the method of \citet{argyriou2005learning}  also benefits from searching in a continuous space, it was seen to require significantly more computation time compared to our method. We also showed that our method can successfully deal with high-dimensional kernel parameter spaces, which, at least in our experiments, 
the method of  \citet{argyriou2005learning,argyriou2006dc} had problems with. \todo{Why? and what about the DC variant!?}

The main lesson of our experiments is that the methods that start by discretizing the kernel space without using the data might lose the potential to achieve good performance before any learning happens. 

We think that currently our method is the most efficient method to design data-dependent dictionaries that provide competitive performance. 
It remains an interesting problem to be explored in the future whether there exist methods that are provably efficient and yet their performance remains competitive. Although in this work we directly compared our method to finite-kernel methods, it is also natural to combine dictionary search methods (like ours) with finite-kernel learning methods. However, the thorough investigation of this option remains for future work.

A secondary outcome of our experiments is the observation that although test-set alignment is generally a good indicator of good predictive performance, a larger test-set alignment does not necessarily transform into a smaller misclassification error. Although this is not completely unexpected, we think that it will be important to thoroughly explore the implications of this observation.

\bibliographystyle{apalike}
\bibliography{./references}

\begin{thebibliography}{}

\bibitem[Argyriou et~al., 2006]{argyriou2006dc}
Argyriou, A., Hauser, R., Micchelli, C., and Pontil, M. (2006).
\newblock {A DC-programming algorithm for kernel selection}.
\newblock In {\em Proceedings of the 23rd international conference on Machine
  learning}, pages 41--48.

\bibitem[Argyriou et~al., 2005]{argyriou2005learning}
Argyriou, A., Micchelli, C., and Pontil, M. (2005).
\newblock Learning convex combinations of continuously parameterized basic
  kernels.
\newblock In {\em Proceedings of the 18th Annual Conference on Learning
  Theory}, pages 338--352.

\bibitem[B{\"u}hlmann and van~de Geer, 2011]{BuhvdG11}
B{\"u}hlmann, P. and van~de Geer, S. (2011).
\newblock {\em Statistics for High-Dimensional Data: Methods, Theory and
  Applications}.
\newblock Springer.

\bibitem[Cortes, 2009]{Cortes09:invited}
Cortes, C. (2009).
\newblock Invited talk: Can learning kernels help performance?
\newblock In {\em ICML '09}, pages 1--1.

\bibitem[Cortes et~al., 2009a]{cortes2009l2}
Cortes, C., Mohri, M., and Rostamizadeh, A. (2009a).
\newblock L2 regularization for learning kernels.
\newblock In {\em Proceedings of the 25th Conference on Uncertainty in
  Artificial Intelligence}, pages 109--116.

\bibitem[Cortes et~al., 2009b]{cortes2009learning}
Cortes, C., Mohri, M., and Rostamizadeh, A. (2009b).
\newblock Learning non-linear combinations of kernels.
\newblock In {\em Advances in Neural Information Processing Systems 22}, pages
  396--404.

\bibitem[Cortes et~al., 2010]{cortes2010two}
Cortes, C., Mohri, M., and Rostamizadeh, A. (2010).
\newblock Two-stage learning kernel algorithms.
\newblock In {\em Proceedings of the 27th International Conference on Machine
  Learning}, pages 239--246.

\bibitem[Cristianini et~al., 2002]{shawe2002kernel}
Cristianini, N., Kandola, J., Elisseeff, A., and Shawe-Taylor, J. (2002).
\newblock On kernel-target alignment.
\newblock In {\em Advances in Neural Information Processing Systems 15}, pages
  367--373. MIT Press.

\bibitem[Frank and Asuncion, 2010]{frank2010uci}
Frank, A. and Asuncion, A. (2010).
\newblock {UCI} machine learning repository.

\bibitem[Hastie et~al., 2001]{hastie2001elements}
Hastie, T., Tibshirani, R., and Friedman, J. (2001).
\newblock {\em The Elements of Statistical Learning}.
\newblock Springer Series in Statistics. Springer-Verlag New York.

\bibitem[Kloft et~al., 2011]{kloft2011lp}
Kloft, M., Brefeld, U., Sonnenburg, S., and Zien, A. (2011).
\newblock $\ell^p$-norm multiple kernel learning.
\newblock {\em Journal of Machine Learning Research}, 12:953--997.

\bibitem[Lanckriet et~al., 2004]{lanckriet2004learning}
Lanckriet, G., Cristianini, N., Bartlett, P., Ghaoui, L., and Jordan, M.
  (2004).
\newblock Learning the kernel matrix with semidefinite programming.
\newblock {\em Journal of Machine Learning Research}, 5:27--72.

\bibitem[{LeCun} and Cortes, 2010]{lecun-mnist-web-2010}
{LeCun}, Y. and Cortes, C. (2010).
\newblock {MNIST} handwritten digit database.

\bibitem[Rakotomamonjy et~al., 2008]{rakotomamonjy2008simplemkl}
Rakotomamonjy, A., Bach, F., Canu, S., and Grandvalet, Y. (2008).
\newblock Simple{MKL}.
\newblock {\em Journal of Machine Learning Research}, 9:2491--2521.

\bibitem[Sonnenburg et~al., 2006]{sonnenburg2006large}
Sonnenburg, S., R{\"a}tsch, G., Sch{\"a}fer, C., and Sch{\"o}lkopf, B. (2006).
\newblock Large scale multiple kernel learning.
\newblock {\em The Journal of Machine Learning Research}, 7:1531--1565.

\end{thebibliography}

\clearpage
\newpage
\appendix
\section{Proofs}
\subsection{Proof of Proposition~\ref{prop:dd}}
First, notice that the limit in~\eqref{eq:sigmatdef} is a directional derivative, $D_{\kappa_\sigma} f(k^{t-1})$.
By the chain rule,
\[
D_{\kappa_\sigma} f(k^{t-1})
= \ip{ \, K(\kappa_\sigma), \, F'_c( K(k^{t-1}) ) \,}_F\,,
\]
where, for convenience, we defined $F_c(K) = A_c(K,\hat{K}^*)$.
Define
\[
F(K) = \ip{K,\hat{K}^*_c}_F/(\norm{K}_F \norm{\hat{K}^*_c}_F)
\]
so that $F_c(K) = F( K_c )$.
Some calculations give that
\[
F'(K) = \frac{ \hat{K}^*_c - \norm{K}_F^{-2}\ip{ K, \hat{K}^*_c}_F\, K }
				   { \norm{K}_F \norm{\hat{K}^*_c}_F}\,
\]
(which is the function defined in~\eqref{eq:Fder}).
We claim that the following holds:
\begin{lemma}
$F_c'(K) = C_n F'( K_c ) C_n$.
\end{lemma}
\begin{proof}
By the definition of derivatives, as $H\ra 0$,
\[F(K+H)-F(K) = \ip{ F'(K),H }_F + o(\|H\|).\]
Also,
 \[F_c(K+H)-F_c(K) = \ip{ F'_c(K),H }_F + o(\|H\|).\]
Now,
\begin{align*}
F_c(K+H) - F_c(K) &= F(C_nK C_n+C_n H C_n) - F(C_n K C_n)\\
&=\ip{F'(K_c), C_n H C_n }_F + o(\|H\|)\\
&=\ip{C_n F'(K_c) C_n, H }_F + o(\|H\|),
\end{align*}
 where the last property follows from the cyclic property of trace.
Therefore, by the uniqueness of derivative, $F_c'(K) = C_n F'(K_c) C_n$.
\end{proof}
Now, notice that $C_n F'(K_c) C_n = F'(K_c)$.
Thus, we see that the value of $\sigma_t^*$ can be obtained by
\[
\sigma_t^* = \argmax_{\sigma\in \Sigma} \,
	\ipL{\, K(\kappa_\sigma), F'(\, (K(k^{t-1}))_c ) \, }_F\,,
\]
which was the statement to be proved.
\subsection{Proof of Proposition~\ref{prop:pp2}}
Let $g(\eta) = f(k^{t-1}+\eta \kappa_{\sigma_t^*})$. Using the definition of $f$, we find that with some constant $\rho>0$,
\[
g(\eta) =\rho\, \frac{a+b \eta }{(c+2d\eta+e\eta^2)^{1/2}}.
\]
Notice that here the denominator is bounded away from zero (this follows from the form of the denominator of $f$).
In particular, $e>0$.
Further,
\begin{align}
\label{eq:glim}
\lim_{\eta\ra \infty} g(\eta) = - \lim_{\eta\ra -\infty} g(\eta) = \rho \frac{ b }{ \sqrt{e} }.
\end{align}
Taking the derivative of $g$ we find that
\[
g'(\eta) = \rho \, \frac{bc-ad+(bd-ae)\eta}{(c+2d\eta + e\eta^2)^{3/2}}.
\]
Therefore, $g'$ has at most one root and $g$ has at most one global extremum, from which the result follows by solving for the root of $g'$ (if $g'$ does not have a root, $g$ is constant).
\section{Details of the numerical experiments}
In this section we provide further details and data for the numerical results.
\subsection{Non-Convexity Issue}
\label{sec:non_convexity}
As we mentioned in Section~\ref{sec:new_method}, our algorithm may need to solve a non-convex optimization problem in each iteration to find the best kernel parameter. Here, we explore this problem numerically, by plotting the function to be optimized in the case of a Gaussian kernel with a single bandwidth parameter.
In particular, we plotted the objective function of Equation~\ref{eq:sigma_star} with its sign flipped, therefore we are interested in the local minima of function $h(\sigma) = -\ipL{\, K(\kappa_\sigma), F'(\, (K(k^{t-1}))_c ) \, }_F$, see
 Figure \ref{fig:result_sigma}.
The function $h$ is shown for some iterations of some of the tasks from both the MNIST and the UCI Letter experiments. The number inside parentheses in the caption specifies the corresponding iteration of the algorithm. On these plots, the objective function does not have more than 2 local minima. Although in some cases the functions have some steep parts (at the scales shown), their optimization does not seem very difficult.
\begin{figure*}[t]
\vspace{.3in}
\centerline{\includegraphics[width=0.8\linewidth]{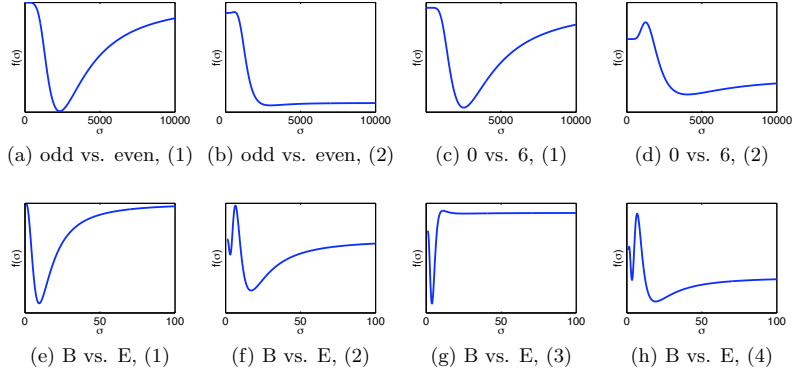}}
\vspace{.3in}
\caption{The flipped objective function underlying~\eqref{eq:sigma_star} as a function of $\sigma$, the parameter of a Gaussian kernel in selected MNIST and UCI Letter problems. Our algorithm needs to find the minimum of
these functions (and similar ones).}
\label{fig:result_sigma}
\end{figure*}


\subsection{Details of the 50-dimensional synthetic dataset experiment}
\label{apx:50d}
The 1-dimensional version of our algorithm, CA-1D, and the CR method, employ Matlab's \textsf{fmincon} function with multiple restarts from the set $10^{\{-3,\ldots,5\}}$, to choose the kernel parameters. The multi-dimensional version of our algorithm, CA-nD, uses \textsf{fmincon} only once, since in this particular example the search method runs on a $50$-dimensional search space, which makes the search an expensive operation. The starting point of the CA-nD method is a vector of equal elements where this element is the weighted average of the kernel parameters found by the CA-1D method, weighted by the coefficient of the corresponding kernels.

The soft margin SVM regularization parameter is tuned from the set $10^{\{-5,-4.5,\ldots,4.5,5\}}$ using an independent validation set with $1000$ instances. We also tuned the value of the regularization parameter in Equation~\eqref{eq:new_sigma_star} from $10^{\{-5,\ldots,14\}}$ using the same validation set (the best value of $\lambda$ is the one that achieves the highest value of alignment on the validation set). We decided to use a large validation set, following essentially the practice of \citet[Section 6.1]{kloft2011lp}, to make sure that in the experiments reasonably good regularization parameters are used, i.e., to factor out the choice of the regularization parameters. This might bias our results towards CA-nD, as compared to CA-1D, though similar results were achieved with a smaller validation set of size $200$. As a final detail note that D1, D2 and CR also use the validation set for choosing the value of their regularization factor, and together with the regularizer, the weights also. Hence, their results might also be positively biased (though we don't think this is significant, in this case).

The running times shown in Figure~\ref{fig:my_toy_experiment} include everything from the beginning to the end, i.e., from learning the kernels to training the final classifiers (the extra cross-validation step is what makes CA-nD expensive).

Figure~\ref{fig:MyToyExperiment_Alignment} shows the (centered) alignment values for the learned kernels (on the test data) as a function of the relevance parameter $\gamma$.  It can be readily seen that the multi-dimensional method has a real-edge over the other methods when the number of irrelevant features is large, in terms of kernel alignment. As seen on Figure~\ref{fig:MyToyExperiment_Alignment}, this edge is also transformed into an edge in terms of the test-set performance. Note also that the discretization is fine enough so that the alignment maximizing finite kernel learning method DA can  achieve the same alignment as the method CA-1D.
\begin{figure}[t]
\vspace{.3in}
\centerline{\includegraphics[width=0.5\linewidth]{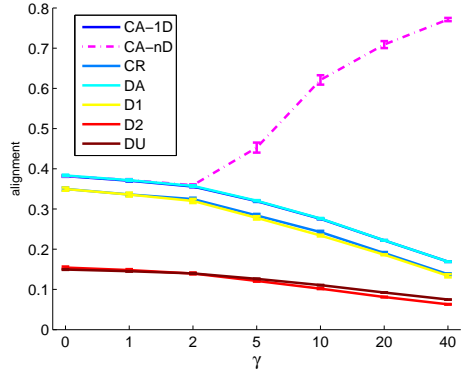}}
\vspace{.3in}
\caption{Alignment values in the $50$-dimensional synthetic dataset experiment.}
\label{fig:MyToyExperiment_Alignment}
\end{figure}

\subsection{Detailed results for the real datasets}
\label{apx:realdata}
\begin{figure*}[htbp]
\vspace{.3in}
\centerline{\includegraphics[width=0.8\linewidth]{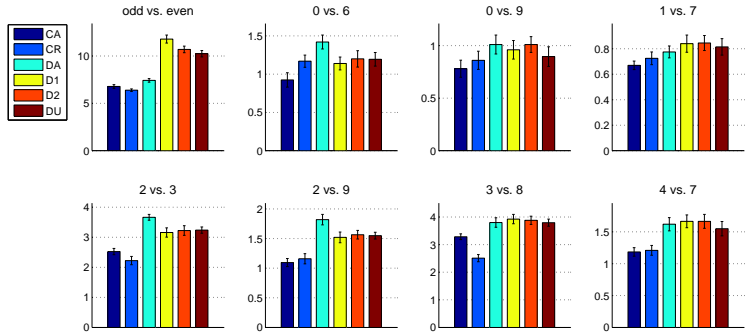}}
\vspace{.3in}
\caption{Misclassification percentages in different tasks of the MNIST dataset.}
\label{fig:result_new_mnist_testerr}
\end{figure*}


\begin{figure*}[htbp]
\vspace{.3in}
\centerline{\includegraphics[width=0.8\linewidth]{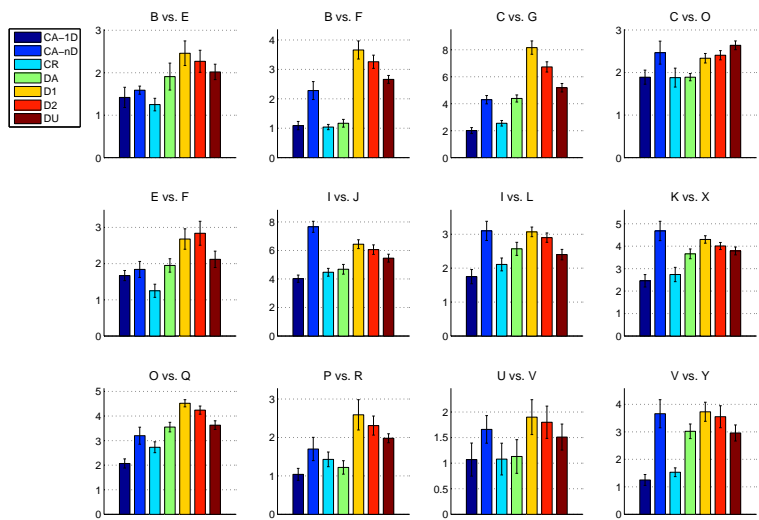}}
\vspace{.3in}
\caption{Misclassification percentages in different tasks of the UCI Letter recognition dataset.}
\label{fig:new_letter_testerr}
\end{figure*}

\begin{figure*}[htbp]
\vspace{.3in}
\centerline{\includegraphics[width=0.8\linewidth]{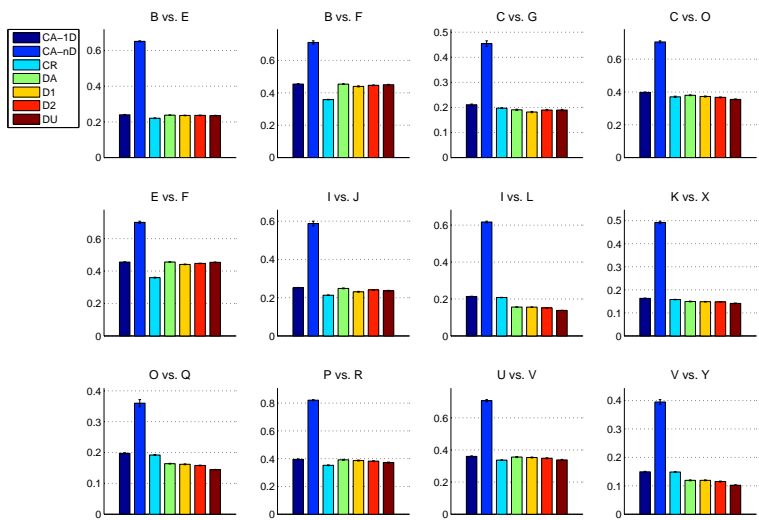}}
\vspace{.3in}
\caption{Alignment values in different tasks of the UCI Letter recognition dataset.}
\label{fig:new_letter_testalignment}
\end{figure*}

\begin{table*}[ht]
\caption{Datasets used in the experiments}
\label{tab:table_label}
\begin{center}
\begin{tabular}{|l|l|l|l|l|l|}
  \hline
	\textbf{Dataset} & \textbf{\# features} & \textbf{\# instances} & \textbf{Training size} & \textbf{Validation size} & \textbf{Test size} \\ \hline
Banana & $2$ & $5300$ & $500$ & $1000$ & $2000$ \\ \hline
Breast Cancer & $9$ & $263$ & $52$ & $78$ & $133$ \\ \hline
Diabetes & $8$ & $768$ & $153$ & $230$ & $385$ \\ \hline
German & $20$ & $1000$ & $200$ & $300$ & $500$ \\ \hline
Heart & $13$ & $270$ & $54$ & $81$ & $135$ \\ \hline
Image Segmentation & $18$ & $2086$ & $400$ & $600$ & $1000$ \\ \hline
Ringnorm & $20$ & $7400$ & $500$ & $1000$ & $2000$ \\ \hline
Sonar & $60$ & $208$ & $41$ & $62$ & $105$ \\ \hline
Splice & $60$ & $2991$ & $500$ & $1000$ & $1491$ \\ \hline
Thyroid & $5$ & $215$ & $43$ & $64$ & $108$ \\ \hline
Waveform & $21$ & $5000$ & $500$ & $1000$ & $2000$ \\ \hline
\end{tabular}
\end{center}
\end{table*}

\begin{figure*}[htbp]
\vspace{.3in}
\centerline{\includegraphics[width=0.8\linewidth]{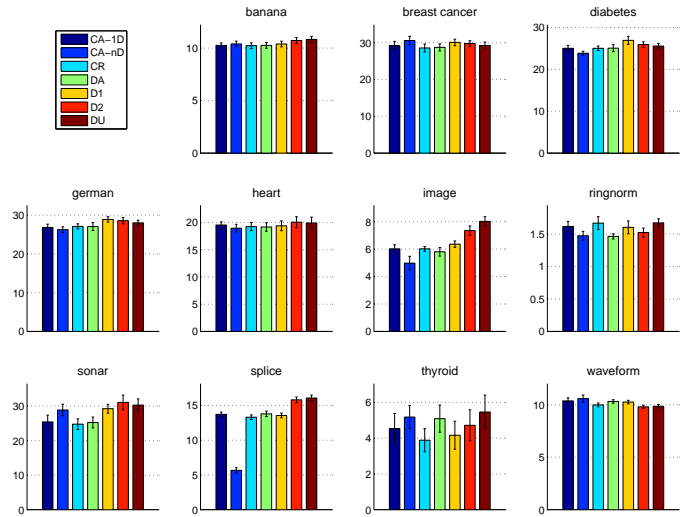}}
\vspace{.3in}
\caption{Misclassification percentages obtained in 11 datasets.}
\label{fig:11_realdata_testerr}
\end{figure*}

\end{document}